\documentclass[letterpaper, 10 pt, journal, twoside]{IEEEtran}

\usepackage{amsmath} 
\usepackage{amssymb}  

\usepackage{amsthm}
\usepackage{cite}

\usepackage{algorithm}
\usepackage[noend]{algpseudocode}
\usepackage{color}
\usepackage{graphicx}
\usepackage{booktabs}
\usepackage{xcolor}
\usepackage{siunitx}
\usepackage[hidelinks]{hyperref}
\usepackage{bm}
\usepackage{threeparttable}

\newtheorem{thm}{Theorem}[section]

\newtheorem{rem}{Remark}

\DeclareMathOperator*{\softmin}{softmin}

\usepackage{geometry}
\begin{document}
\title{Sailing Through Point Clouds: Safe Navigation Using Point Cloud Based Control Barrier Functions}

\author{Bolun Dai$^{*1}$,~\IEEEmembership{Student Member,~IEEE}, Rooholla Khorrambakht$^{*1}$,~\IEEEmembership{Student Member,~IEEE},\\ Prashanth Krishnamurthy$^{1}$,~\IEEEmembership{Member,~IEEE}, Farshad Khorrami$^{1}$,~\IEEEmembership{Senior Member,~IEEE}
\thanks{$^{*}$ Equal Contribution.}
\thanks{Manuscript received: March 26th, 2024; Revised: June 17th, 2024; Accepted: July 14th, 2024.}%
\thanks{This paper was recommended for publication by Editor A. Bera upon evaluation of the Associate Editor and Reviewers' comments. This work was supported in part by the ARO under grant W911NF-22-1-0028 and the NYUAD Center for Artificial Intelligence and Robotics (CAIR), funded by Tamkeen under the NYUAD Research Institute Award CG010.}%
\thanks{$^{1}$Bolun Dai, Rooholla Khorrambakht, Prashanth Krishnamurthy, and Farshad Khorrami are with Control/Robotics Research Laboratory, Electrical~\&~Computer Engineering Department, Tandon School of Engineering, New York University, Brooklyn, NY 11201
{\tt\footnotesize \{bolundai, rk4342, prashanth.krishnamurthy, khorrami\}@nyu.edu}}%
\thanks{Digital Object Identifier (DOI): see top of this page.}
}

\markboth{IEEE Robotics and Automation Letters. Preprint Version. Accepted July, 2024}
{Dai \MakeLowercase{\textit{et al.}}: Sailing Through Point Clouds: Safe Navigation Using Point Cloud Based Control Barrier Functions}

\maketitle
\urldef{\video}\url{https://youtu.be/P9NPv1f3kXQ}
\begin{abstract}
The capability to navigate safely in an unstructured environment is crucial when deploying robotic systems in real-world scenarios. Recently, control barrier function (CBF) based approaches have been highly effective in synthesizing safety-critical controllers. In this work, we propose a novel CBF-based local planner comprised of two components: \textbf{\textit{Vessel}} and \textbf{\textit{Mariner}}. The \textbf{\textit{Vessel}} is a novel scaling factor based CBF formulation that synthesizes CBFs using only point cloud data. The \textbf{\textit{Mariner}} is a CBF-based preview control framework that is used to mitigate getting stuck in spurious equilibria during navigation. To demonstrate the efficacy of our proposed approach, we first compare the proposed point cloud based CBF formulation with other point cloud based CBF formulations. Then, we demonstrate the performance of our proposed approach and its integration with global planners using experimental studies on the Unitree B1 and Unitree Go2 quadruped robots in various environments.
\end{abstract}
\begin{IEEEkeywords}
Robot Safety; Collision Avoidance; Motion and Path Planning.
\end{IEEEkeywords}
\IEEEpeerreviewmaketitle
\section{Introduction}
\label{sec:introduction}
\IEEEPARstart{M}{obile} robots have been deployed to navigate unstructured environments autonomously for various operations~\cite{JianYLLLWL23, UnluGCTK24, DaiKPK23, RawlingsMD17, AmesCENST19, GoncalvesKTK24, ThirugnanamZS22, ZengZS21}. Traditionally, work has been done in generating motion plans for navigation using pre-built maps with global motion planners like rapidly exploring random trees (RRTs)~\cite{LaValle06}. In~\cite{LiuWMSBTK17}, a safe corridor based approach is used to reduce the safe set to polytopes along a pre-planned path, which enables online replanning. However, when the map has high granularity as well as when the map is constructed/updated in real-time, it would be costly to frequently replan using these global planners, creating the need for a local planner that directly operates on sensor measurements while ensuring safety. Recently, control barrier functions (CBFs)~\cite{AmesCENST19} have been gaining popularity in synthesizing safe control actions. One of the main benefits of CBF-based approaches is that they can transform nonlinear and nonconvex constraints into linear ones, greatly increasing the computation speed. 

One main challenge in utilizing CBF-based methods for navigation is the synthesis of CBFs~\cite{DaiKK22, DaiHKK23}. Recently, work~\cite{DaiKKGTK23, DaiKKK23, WeiDKKK24} has been done using differentiable optimization based growth distance~\cite{OngG96, GilbertO94} CBFs to model robots and their surroundings using a wider range of geometric primitives~\cite{TracyHM22}. Although~\cite{DaiKKGTK23} effectively synthesizes safe controls, it requires an extra step of decomposing the environment into a collection of primitives. In navigation tasks, the environment is commonly represented by point clouds or their derivatives, such as voxel grids and octomaps. There are recent works on synthesizing CBFs directly from point cloud data, such as~\cite{JianYLLLWL23, UnluGCTK24, ZhangGF23, CosnerRMUYAB22, SingletaryKBBTA21}. In~\cite{JianYLLLWL23}, the CBF formulation relies on encapsulating the obstacles with ellipsoids while representing the robot with a sphere (possibly causing overbounding). In~\cite{UnluGCTK24}, the obstacles are modeled directly using point clouds, and a CBF formulation is used based on the distance between a point and a smoothed rectangle. In~\cite{CosnerRMUYAB22, SingletaryKBBTA21}, the CBFs are modeled between a point cloud and a spherical robot, however, their CBF formulation is not continuously differentiable due the use of the minimum function. In~\cite{ZhangGF23}, a point cloud based CBF is learned, however, it is unclear whether the learned CBF is valid or can scale to real-life point cloud measurements. Compared with prior point cloud based CBF approaches, the CBF formulation proposed in this paper is theoretically valid and directly operates on the point clouds or similar representations (e.g., voxels or octomaps) without any post-processing (which is prone to failure with non-zero probability). Furthermore, our method models the robot using more general shapes, i.e., higher-order ellipsoids, which is a superset of prior works~\cite{JianYLLLWL23, UnluGCTK24, CosnerRMUYAB22, SingletaryKBBTA21}.

Once a valid CBF is synthesized, a CBF-based quadratic program (CBFQP) is commonly used to generate safe actions, which can be seen as model predictive control (MPC) with a one-timestep preview horizon. This myopic nature makes CBF-based methods susceptible to getting stuck in spurious equilibria~\cite{GoncalvesKTK24}, which makes it challenging to utilize CBF-based methods for navigation tasks. In~\cite{GoncalvesKTK24}, a circulation constraint was introduced to avoid spurious equilibria by inducing the robot to circulate obstacles. There are recent efforts in combining CBFs with MPC~\cite{JianYLLLWL23, ZengZS21}. Although the CBF constraint is linear for one timestep, it becomes nonlinear and nonconvex for longer preview horizons. This often makes the CBF-MPC a nonlinear MPC (NMPC), which is computationally expensive to solve.

To address the various challenges outlined above, this paper proposes a novel point cloud CBF based local planner. The main contributions of this paper are as follows: 
\begin{enumerate}
    \item proposed {\it\textbf{Vessel}}: a novel growth distance based point cloud CBF formulation;
    \item proposed {\it\textbf{Mariner}}: a novel CBF-based local preview controller that can better mitigate the spurious equilibrium issue;
    \item empirically validated the proposed safe local planner ({\it\textbf{Vessel}} + {\it\textbf{Mariner}}) in simulation and the real world on the Unitree B1 and Unitree Go2 quadruped robots;
    \item integrated the proposed safe local planner with a global planner and validated its performance in indoor environments.
\end{enumerate}
Compared with previous works, the main benefits of our proposed method are:
\begin{enumerate}
    \item the CBF formulation (Vessel) is more general and does not require post-processing of the point cloud data;
    \item the local planner (Mariner) can replan in real-time and directly operates on sensor readings without relying on a mapping algorithm;
    \item the combination of Mariner and a global planner (e.g., RRT$^\star$) makes the generated motion plan biased towards shorter paths and straight lines.
\end{enumerate}
The remainder of this paper is structured as follows. Section~\ref{sec:preliminaries} briefly reviews CBFs. In Section~\ref{sec:problem_formulation}, the safe indoor navigation problem is formulated. In Section~\ref{sec:method}, we present the Vessel and Mariner formulation. In Section~\ref{sec:experiments}, we first compare our proposed method with existing point cloud based CBF methods for safe navigation; then, we perform ablation studies to show the effectiveness of each component of our proposed approach. Additionally, real-world experiments were performed to demonstrate the effectiveness of our proposed approach. In Section~\ref{sec:limitations}, we discuss the limitations of our proposed approach. Finally, in Section~\ref{sec:conclusion}, we conclude the paper with a discussion on future directions.
\section{Preliminaries}
\label{sec:preliminaries}
In this section, we provide a brief introduction to CBFs~\cite{AmesCENST19} and CBF-based safe control synthesis. Consider a control affine system
\begin{equation}
    \dot{\bm{x}} = \bm{F}(\bm{x}) + \bm{G}(\bm{x})\bm{u}
\label{eq:control_affine_sys}
\end{equation}
where the state is represented as $\bm{x} \in \mathbb{R}^n$, the control as $\bm{u} \in \mathbb{R}^m$, the drift as $\bm{F}: \mathbb{R}^n \rightarrow \mathbb{R}^n$, and the control matrix as $\bm{G}: \mathbb{R}^n \rightarrow \mathbb{R}^{n \times m}$. Following~\cite{AmesCENST19}, define two sets $\mathcal{C}$ and $\mathcal{D}$ where the relationship between the two sets is $\mathcal{C} \subset \mathcal{D} \subset \mathbb{R}^n$. Let $\bm{h}:\mathcal{D} \rightarrow \mathbb{R}$ be a continuously differentiable function that has $\mathcal{C}$ as its 0-superlevel set, i.e., 
\begin{equation}
    \mathcal{C} = \{\bm{x} \mid \bm{h}(\bm{x}) \geq 0, \bm{x}\in\mathbb{R}^n\}
\end{equation} 
and $\partial\bm{h}/\partial \bm{x} \neq 0$ for all $\bm{x}\in\partial\mathcal{C}$, where $\partial\mathcal{C} \subset \mathbb{R}^n$ represents the boundary of $\mathcal{C}$. Then, if 
\begin{equation}
    \sup_{\bm{u}\in\mathcal{U}}\Big[\frac{\partial\bm{h}(\bm{x})}{\partial \bm{x}}\Big(\bm{F}(\bm{x}) + \bm{G}(\bm{x})\bm{u}\Big)\Big] \geq -\Gamma(\bm{h}(\bm{x}))
    \label{eq:CBF_constraint}
\end{equation}
where $\mathcal{U}\subset\mathbb{R}^m$ being the set of admissible controls, holds for all $\bm{x}\in\mathcal{D}$ with $\Gamma: \mathbb{R}\rightarrow\mathbb{R}$ being an extended class $\mathcal{K}_\infty$ function\footnote{Extended class $\mathcal{K}_\infty$ functions are strictly increasing with $\Gamma(0) = 0$ and $\lim_{a\rightarrow\infty}\Gamma(a) = \infty$.}, $\bm{h}$ is a CBF on $\mathcal{C}$ and renders the control-invariance of the safe set $\mathcal{C}$ as follows. The CBF constraint is utilized to formulate the safe control problem as a CBFQP~\cite{AmesCENST19}, which has the form of
\begin{align}
\label{eq:cbfqp}
    \min_{\bm{u}\in\mathcal{U}}\ &\ \|\bm{u}_\text{ref} - \bm{u}\|_2^2\\
    \text{subject to}\ &\ \frac{\partial\bm{h}(\bm{x})}{\partial \bm{x}}\Big(\bm{F}(\bm{x}) + \bm{G}(\bm{x})\bm{u}\Big) \geq -\Gamma(\bm{h}(\bm{x}))\nonumber
\end{align}
with $\bm{u}_\text{ref} \in \mathbb{R}^m$ being a reference control action obtained using a performance controller~\cite{AmesCENST19}. The solution to the CBFQP is a safe controller that can keep the system in the safe set $\mathcal{C}$. The performance controller can be obtained using a variety of methods. Common performance controllers include proportional-derivative (PD), MPC, and control Lyapunov function (CLF) based controllers.
\section{Problem Formulation}
\label{sec:problem_formulation}
This work addresses the local planning problem for safe navigation of velocity-controlled robotic systems with dynamics $\dot{\bm{x}} = \bm{G(x)u}$. We make the following assumptions.
\begin{enumerate}
    \item The robot is given a point cloud representation of the environment in the body frame through an installed LiDAR sensor or a perception pipeline (e.g., point cloud generated from Octomaps). 
    \item The existence of a point in space implies an obstacle at that point that has to be avoided.
    \item The provided point cloud is clean, and outliers are removed (or a high-quality LiDAR sensor is used).
    \item The existence of a state estimator for computing the reference controller command and updating the local target. Note that our body-centric CBFQP filter does not require this estimator. 
\end{enumerate}
Given a 2D occupancy map of the environment, we formulate a complete navigation pipeline by combining our local planner with off-the-shelf global path planners (e.g., RRT). 
\begin{rem}
    An example application scenario for this pipeline is when we have map (generated a priori or online) of an environment and want to navigate through this space while being safe against changes (e.g., adding new objects such as cans or boxes) without requiring the map or the global planner to be updated frequently.
\end{rem}
\section{Method}
\label{sec:method}
This section presents our proposed local planner and how it integrates into a navigation pipeline.

\subsection{Vessel: Point Cloud Based Control Barrier Function}
\label{sec:method_cbf}
\begin{figure*}[t!]
    \centering
    \includegraphics[width=0.95\textwidth]{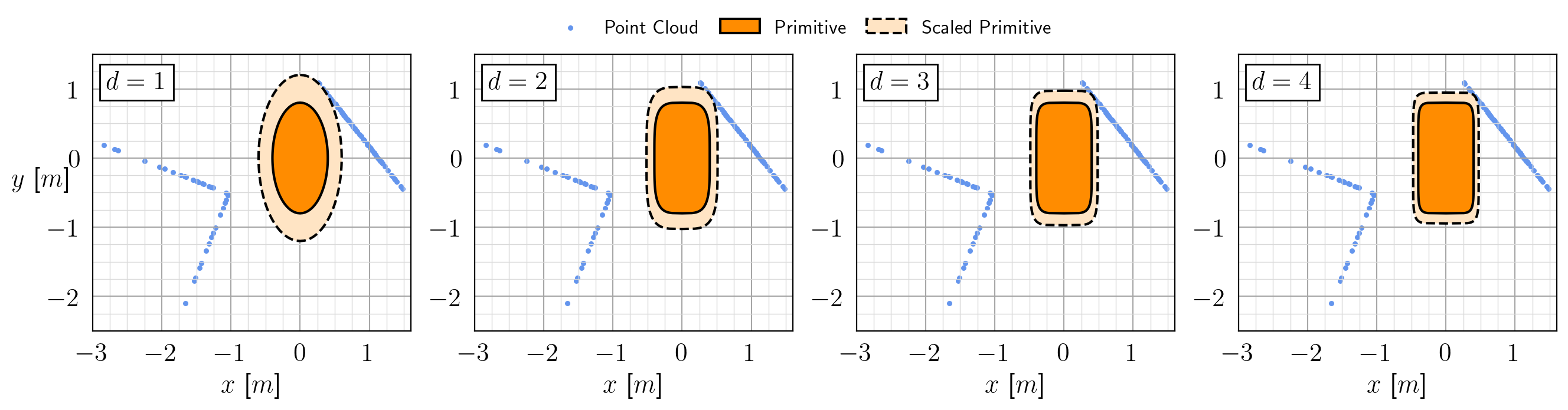}
    \caption{Visualization of the proposed point cloud CBF in a 2D scenario for different versions of higher-order ellipsoids. The blue dots represent the measured point cloud. The dark orange higher-order ellipse represents the unscaled primitive defined in~\eqref{eq:ellipsoid}, and the light orange higher-order ellipse represents the primitive after it is scaled by the scaling factor. The order of the ellipse is given in the upper left corner.}
    \label{fig:cbf_concept_visualization}
\end{figure*}
Assume the robot, in its body frame, is encapsulated by an ellipsoid $\mathcal{E}$ defined as
\begin{equation}
\label{eq:ellipsoid}
   \mathcal{E} = \Big\{\bm{p}\ \Big|\ \bm{p}^\top\bm{P}\bm{p} \leq 1\Big\}
\end{equation}
 and the matrix $\bm{P}$ given by
\begin{equation}
    \bm{P} = \begin{bmatrix}
        1/a^2 & 0 & 0\\
        0 & 1/b^2 & 0\\
        0 & 0 & 1/c^2
    \end{bmatrix}.
\end{equation}
with $a, b, c \in \mathbb{R}_+$ denoting the lengths of its semi-axes.
\begin{rem}
    We call this ellipsoid (including the higher-order ellipsoids as defined in~\eqref{eq:higher_order_ellipsoid}) a \textbf{\textit{Vessel}} that contains the robot.
\end{rem}
Let the obstacles be given as point cloud $\mathcal{P} \subset \mathbb{R}^3$ in the body frame of $\mathcal{E}$. Then, for each point $p_j \in \mathcal{P}$, compute the uniform scaling factor $\bm{s}_{j} \in \mathbb{R}_+$ between the ellipsoids and $\bm{p}_j$
\begin{equation}
    \bm{s}_{j}^2 = \bm{p}_j^\top\bm{P}\bm{p}_j.
\end{equation}
As seen in Fig.~\ref{fig:cbf_concept_visualization}, a larger-than-one scaling factor implies that the robot is not in collision with the point cloud. Thus, we can specify the safe set as the set of states where $\bm{s}_{j}^2 \geq 1$. Define $\bm{\alpha}_{j}$ and $\bm{h}_{j}$ as
\begin{equation}
    \bm{\alpha}_{j} = \bm{s}_{j}^2, \quad \bm{h}_{j} = \bm{\alpha}_{j} - \bm{\beta}
\end{equation}
where $\bm{\beta} \geq 1$. Then, the CBF between point cloud $\mathcal{P}$ and $\mathcal{E}$ can be computed as
\begin{equation}
    \bm{h} = \min_{j}{\bm{h}_{j}}.
\end{equation}
\begin{rem}
    From the derivations in~\cite{DaiKKK23}, we know that by ensuring safety with respect to the point with the smallest CBF value, we can ensure safety with respect to all points in the point cloud. 
\end{rem}
A visualization of the effect of scaling an ellipsoid using the smallest scaling factor is given in Fig.~\ref{fig:cbf_concept_visualization}. This formulation can be extended to higher-order ellipsoids defined as
\begin{equation}
\label{eq:higher_order_ellipsoid}
    \mathcal{E}^{\bm{d}} = \Big\{\bm{p}\Big|\bm{p}_{\bm{d}}(\bm{p}, \bm{d})^\top\bm{P}^{\bm{d}}\bm{p}_{\bm{d}}(\bm{p}, \bm{d}) \leq 1\Big\}
\end{equation}
with $\bm{d}\in\mathbb{N}$ representing the order of the ellipse and $\bm{p}_{\bm{d}}: \mathbb{R}^3\times\mathbb{N} \rightarrow \mathbb{R}^n$  defined as
\begin{equation}
\label{eq:high-order-scaling-factor}
    \bm{p}_{\bm{d}}(\bm{p}, \bm{d}) = \begin{bmatrix}
        x^{\bm{d}}\\
        y^{\bm{d}}\\
        z^{\bm{d}}
    \end{bmatrix}\in\mathbb{R}^3,\ \mathrm{with}\ \bm{p} = \begin{bmatrix}
        x\\
        y\\
        z
    \end{bmatrix}\in\mathbb{R}^3.
\end{equation}
For higher-order ellipsoids, the scaling factor and CBF would be computed as
\begingroup
\allowdisplaybreaks
\begin{subequations}
\begin{align}
    \bm{s}_{j} &= \sqrt[\bm{d}]{\bm{p}_{\bm{d}}(\bm{p}_j, \bm{d})^\top\bm{P}_i^{\bm{d}}\bm{p}_{\bm{d}}(\bm{p}_j, \bm{d})}\\
    \bm{\alpha}_{j} &= \bm{s}_{j}^{\bm{d}}\\
    \label{eq:high-order-cbf}
    \bm{h}_{j} &= \bm{\alpha}_{j} - \bm{\beta}.
\end{align}
\end{subequations}
\endgroup
\begin{rem}
    In this paper, ellipsoids refer to a higher-order ellipsoid of order one. The order for higher-order ellipsoids with $\bm{d} > 1$ will be specified.
\end{rem}
A visualization of using different orders of higher-order ellipsoids is also given in Fig.~\ref{fig:cbf_concept_visualization}. Given that the minimum function is not differentiable, we follow the formulation in~\cite{GoncalvesKTK24} and adopt a soft version, i.e., the $\mathrm{softmin}$ function, which is defined as
\begin{equation}
\label{eq:cbf_softmin_formulation}
    \softmin_{j}\bm{h}_{j} = -\delta\ln\Bigg(\frac{1}{N}\sum_{j=1}^{N}{e^{-\bm{h}_{j}/\delta}}\Bigg)
\end{equation}
where $N \in \mathbb{N}$ represents the number of points in the point cloud $\mathcal{P}$ and $\delta\in\mathbb{R}_+$. The $\mathrm{softmin}$ approximates the minimum function when $\delta$ is small. The expression in~\eqref{eq:cbf_softmin_formulation} is susceptible to numerical instability, which we address by adopting the following equivalent form~\cite{GoncalvesKTK24}
\begin{subequations}
\begin{align}
    \softmin_{j}\bm{h}_{j} &= \min_{j}\bm{h}_{j} + \softmin_{j}\widehat{\bm{h}}_{j}\\
    \widehat{\bm{h}}_{j} &= \bm{h}_{j} - \min_{k}\bm{h}_{k}.
\end{align}
\end{subequations}
With this, the final form of our CBF formulation becomes
\begin{equation}
    \bm{h} = \min_{j}\bm{h}_{j} - \delta\ln\Bigg(\frac{1}{N}\sum_{j=1}^{N}{e^{-\widehat{\bm{h}}_{j}/\delta}}\Bigg).
    \label{eq:point_cloud_cbf}
\end{equation}
The partial derivative of the CBF with respect to the position and orientation of the ellipsoid can be computed as
\begin{equation}
    \frac{\partial\bm{h}}{\partial(\bm{r}, \bm{q})} = \frac{\displaystyle\sum_{j=1}^{N}{\Bigg[e^{-\widehat{\bm{h}}_{j}/\delta}\frac{\partial\bm{h}_{j}}{\partial(\bm{r}, \bm{q})}}\Bigg]}{\displaystyle\sum_{j=1}^{N}{e^{-\widehat{\bm{h}}_{j}/\delta}}} \in \mathbb{R}^7
\end{equation}
where $\bm{r}\in\mathbb{R}^3$ and $\bm{q}\in\mathbb{H}$ ($\mathbb{H}$ represent the Hamiltonian) represents the position of and orientation as quaternions of $\mathcal{E}$ in the world frame, respectively. Next, we present the theoretical results that show the proposed point cloud CBF formulation in~\eqref{eq:point_cloud_cbf} is a valid CBF.

\begin{thm}
\label{thm:cbf_cont_diff}
    For a point cloud $\mathcal{P}$, the CBF proposed in~\eqref{eq:point_cloud_cbf} is continuously differentiable with respect to the position and orientation of its encapsulating higher-order ellipsoid.
\end{thm}
\begin{proof}
    The proof utilizes the following properties: for two continuously differentiable functions $\bm{f}_1:\mathbb{R}^n\rightarrow\mathbb{R}$ and $\bm{f}_2:\mathbb{R}^n\rightarrow\mathbb{R}$, we have $\bm{f}_1(\bm{f}_2(\bm{x}))$, $\bm{f}_1(\bm{x}) + \bm{f}_2(\bm{x})$, and $\bm{f}_1(\bm{x})\bm{f}_2(\bm{x})$ are all continuously differentiable functions. For a point $\bm{p}_j$ measured in the body frame of $\mathcal{E}$, the corresponding world frame coordinate is
    \begin{equation}
        \bm{p}_{w, j} = \bm{Rp_j + r} = \begin{bmatrix}
            x_{w, j} & y_{w, j} &  z_{w, j}
        \end{bmatrix}^\top.
    \end{equation}
    where $\bm{R}\in\mathrm{SO}(3)$ represents the orientation of $\mathcal{E}$ in the world frame as a rotation matrix. Then, expanding the CBF in~\eqref{eq:high-order-cbf} for an order $\bm{d}$ ellipsoid yields
    \begingroup
    \allowdisplaybreaks
    \begin{align}
        \bm{h}_{j} = &\ \frac{1}{a^{2\bm{d}}}\Big[-(\bm{r}_{x} - x_{w, j})(2\bm{q}_{w}^2 + 2\bm{q}_{x}^2 - 1)\nonumber\\
              &-\ 2(\bm{r}_{y} - y_{w, j})(\bm{q}_{w}\bm{q}_{z} + \bm{q}_{x}\bm{q}_{y})\nonumber\\
              &+\ 2(\bm{r}_{z} - z_{w, j})(\bm{q}_{w}\bm{q}_{y} - \bm{q}_{x}\bm{q}_{z})\Big]^{2\bm{d}}\nonumber\\
              &+\ \frac{1}{b^{2\bm{d}}}\Big[2(\bm{r}_{x} - x_{w, j}) (\bm{q}_{w}\bm{q}_{z} - \bm{q}_{x}\bm{q}_{y})\nonumber\\
              &-\ (\bm{r}_{y} - y_{w, j})(2\bm{q}_{w}^2 + 2\bm{q}_{y}^2 - 1)\nonumber\\
              &-\ 2(\bm{r}_{z} - z_{w, j})(\bm{q}_{w}\bm{q}_{x} + \bm{q}_{y}\bm{q}_{z})\Big]^{2\bm{d}}\nonumber\\
              &+\ \frac{1}{c^{2\bm{d}}}\Big[-2(\bm{r}_{x} - x_{w, j})(\bm{q}_{w}\bm{q}_{y} + \bm{q}_{x}\bm{q}_{z})\nonumber\\
              &+\ 2(\bm{r}_{y} - y_{w, j})(\bm{q}_{w}\bm{q}_{x} + \bm{q}_{y}\bm{q}_{z})\nonumber\\
              &-\ (\bm{r}_{z} - z_{w, j})(2\bm{q}_{w}^2 + 2\bm{q}_{z}^2 - 1)\Big]^{2\bm{d}} - \bm{\beta}.
    \end{align}
    \endgroup
    Then, based on the aforementioned properties, we can see that $\bm{h}_{j}$ is continuously differentiable, with respect to $\bm{r} = [\bm{r}_x, \bm{r}_y, \bm{r}_z]$ and $\bm{q} = [\bm{q}_w, \bm{q}_x, \bm{q}_y, \bm{q}_z]$. 
\end{proof}

\begin{rem}
\label{rem:cbf_valid_u}
    The proposed point cloud CBF formulation is for velocity control, i.e., systems with the dynamics $\dot{\bm{x}} = \bm{G(x)u}$. Thus, $\bm{u} = \mathbf{0}$ is always a feasible solution to the CBFQP within the safe set $\mathcal{C}$.  
\end{rem}

\begin{rem}
\label{rem:cbf_valid_boundary}
    Apart from degenerate cases where the point cloud fully surrounds the robot, there always exists a control input that enables the robot to move away from the obstacles, which means that $\partial\bm{h}/\partial \bm{x} \neq 0$ for all $\bm{x}\in\partial\mathcal{C}$.
\end{rem}

\begin{thm}
    For systems with dynamics $\dot{\bm{x}} = \bm{G(x)u}$, the CBF formulation in~\eqref{eq:point_cloud_cbf} is a valid CBF.
\end{thm}
\begin{proof}
    The proof directly follows from Theorem~\ref{thm:cbf_cont_diff}, Remark~\ref{rem:cbf_valid_u}, and Remark~\ref{rem:cbf_valid_boundary}.
\end{proof}

\subsection{Mariner: CBF-Based Safe Preview Control}
\label{sec:method_preview}
\begin{figure}[t!]
    \centering
    \includegraphics[width=0.45\textwidth]{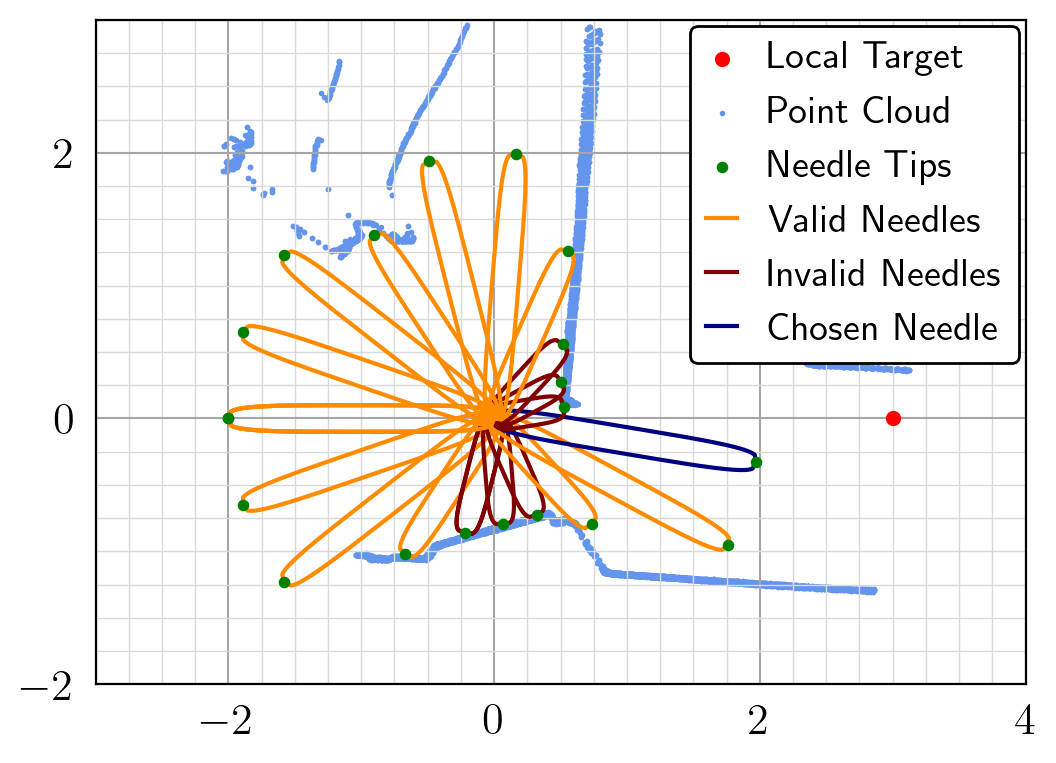}
    \caption{Visualization of the proposed preview control method in a 2D scenario with higher-order ellipsoids with order four.}
    \label{fig:preview_concept_visualization}
\end{figure}
\begin{figure*}
    \centering
    \vspace{0.3cm}
    \includegraphics[width=0.95\textwidth]{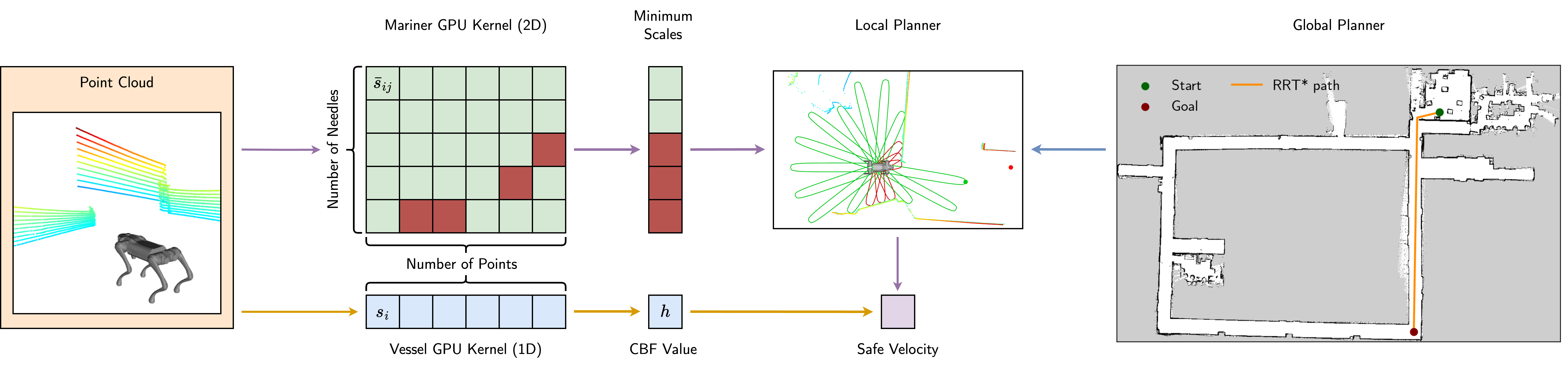}
    \caption{Illustration of the complete motion planning pipeline. The values along the orange arrows are updated at 10~\si{Hz} (LiDAR update frequency), the purple arrows are updated at 2~\si{Hz}, and the blue arrows are updated only once at the beginning of our experiments.}
    \label{fig:complete_pipeline}
\end{figure*}
\begin{figure*}
    \centering
    \includegraphics[width=0.95\textwidth]{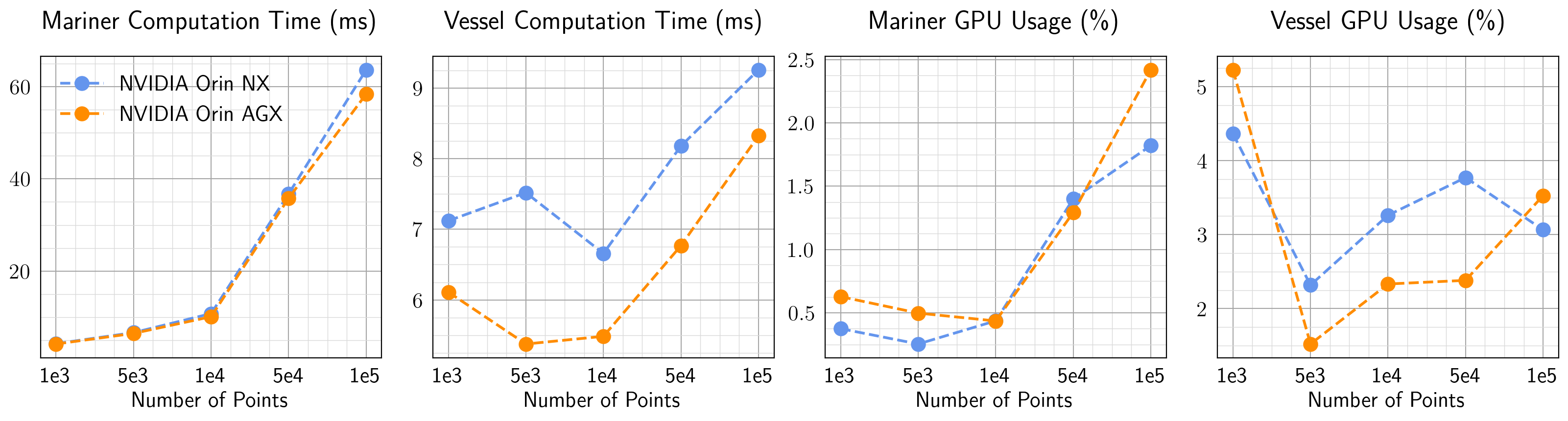}
    \caption{Comparison of the proposed algorithm's computation time and GPU usage on different computers. The GPU usage is computed when running the Mariner at 2~\si{Hz} and the Vessel at 10~\si{Hz}.}
    \label{fig:computation_time}
\end{figure*}
For CBFQP-based methods, one undesirable scenario is when the solution to~\eqref{eq:cbfqp} is $\bm{u} = 0$, which means the robot has fallen into an undesired equilibrium state~\cite{GoncalvesKTK24}. Since the proposed CBF in Section~\ref{sec:method_cbf} is for velocity control, $\bm{u} = 0$ means that the commanded velocity is zero and the robot is stuck at its current position. One way to avoid this issue is to use an MPC-based approach with a sufficiently long preview horizon. However, incorporating the current CBF formulation with an MPC-based approach is computationally expensive and generally creates a nonlinear and nonconvex numerical optimization problem. Therefore, we propose a novel preview controller that, by fixing the shape of the future path, transforms the preview motion planning problem into a collision-free inverse kinematics problem.

The simplification of the motion planning problem by using a fixed set of paths has been studied in dynamic window based approaches~\cite{FoxBT96}. However, in~\cite{FoxBT96}, the paths are generated using a fixed velocity profile, which is less flexible than desired. In this work, we propose \textbf{\textit{Mariner}}, a point cloud CBF based local planner that represents the paths in the configuration space and exploits GPU acceleration to check hundreds of paths (which are referred to as \textit{needles}) in a few milliseconds. Mathematically, the $i$-th needle is represented using a higher-order ellipsoid defined as
\begin{equation}
    \frac{(x - \bar{\bm{s}}_i\bar{a}_i)^{\bar{\bm{d}}}}{\bar{\bm{s}}_i^{\bar{\bm{d}}}\bar{a}_i^{\bar{\bm{d}}}} + \frac{y^{\bar{\bm{d}}}}{\bar{b}_i^{\bar{\bm{d}}}} + \frac{z^{\bar{\bm{d}}}}{\bar{c}_i^{\bar{\bm{d}}}} = 1
\end{equation}
where $\bar{a}_i, \bar{b}_i, \bar{c}_i \in \mathbb{R}$ represents the length of the semi-axes of the unscaled needle, $\bar{\bm{d}}\in\mathbb{N}$ represents the ellipsoid order used to represent the needle, and $\bar{\bm{s}}_i\in\mathbb{R}$ is the scaling factor of the needle. As shown in Fig.~\ref{fig:preview_concept_visualization}, the configuration space is discretized using $n_\text{needle} \in \mathbb{N}$ needles. Each needle is chosen to point towards a fixed angle within the robot's body frame obtained from a distribution function $\mathrm{dist}: \mathbb{N} \rightarrow \mathbb{R}$ that maps the needle index to its corresponding angle. For example, in Fig.~\ref{fig:preview_concept_visualization}, $n_\text{needle} = 19$ and
\begin{equation}
\label{eq:dist_func}
    \mathrm{dist}(i) = \theta_{\text{needle}, i} = 2\pi i/n_\text{needle} - \pi.
\end{equation}
Note that, unlike the Vessel, only the length is extended when the needles are scaled. The body frame of each needle is located at the robot center, with its $z$ axis pointing upwards and its $x$ axis pointing from the robot center to the tip of the needle. To compute the scaling factor $\bar{\bm{s}}_i$, we first rotate the point cloud by the needle angle to get the points within the needle's body frame
\begin{equation}
\label{eq:transform_pcd}
    \bm{p}_j = \begin{bmatrix}
        x_j &y_j & z_j
    \end{bmatrix}^\top = \mathbf{R}(\theta_{\text{needle}, i})^\top\bm{p}_{w, j}
\end{equation}
with $\theta_{\text{needle}, i} \in \mathbb{R}$ represents the angle the $i$-th needle is pointing towards, $\mathbf{R}(\theta_{\text{needle}, i}) \in \mathrm{SO(3)}$ represents the rotation matrix corresponding to the $i$-th needle angle, and $\bm{p}_j, \bm{p}_{w, j} \in \mathbb{R}^3$ representing the position of the $j$-th point in the needle's body frame and world frame, respectively.  For $\bm{p}_j$, define 
\begin{equation}
\label{eq:md}
    m_{ij}^{\bar{\bm{d}}} = 1 - y_j^{\bar{\bm{d}}}/\bar{b}_i^{\bar{\bm{d}}} + z_j^{\bar{\bm{d}}}/\bar{c}_i^{\bar{\bm{d}}}.
\end{equation}
Since the needle can only change its length, not its width or height, we only compute the scale against points with $m_{ij}^{\bar{\bm{d}}} > 0$. With this assumption, we have the scale of the $i$-th needle with respect to the $j$-th point as
\begin{equation}
    \bar{\bm{s}}_{ij} = x_j / [(1 + m_{ij})\bar{a}_i].
\end{equation}
For each needle, we compute the scale with respect to all valid points. Then, the minimum value among all of the scales is chosen as the scale of the needle. The scale is set to a predefined $\bar{\bm{s}}_{\max}$ if no valid points exists or if $\bar{\bm{s}}_i > \bar{\bm{s}}_{\max}$, i.e.,
\begin{equation}
    \bar{\bm{s}}_i = \begin{cases}
        \displaystyle\bar{\bm{s}}_{\max} & \displaystyle\text{if }\min_j{\bar{\bm{s}}_{ij}} > \bar{\bm{s}}_{\max}\text{ or no valid points}\\
        \displaystyle\min_j{\bar{\bm{s}}_{ij}}  & \displaystyle\text{if }\min_j{\bar{\bm{s}}_{ij}} \leq \bar{\bm{s}}_{\max}.
    \end{cases}
\end{equation}
Then, we check the size of the needle scales and only consider needles with $\bar{\bm{s}}_i \geq \bar{\bm{s}}_{\min}$ as valid needles (orange needles shown in Fig.~\ref{fig:preview_concept_visualization}) and treat the rest as invalid needles (red needles shown in Fig.~\ref{fig:preview_concept_visualization}). For all valid needles, we compute the distance between the scaled needle's tip (green dots in Fig.~\ref{fig:preview_concept_visualization}) and the target $\bm{p}^\star \in \mathbb{R}^3$ (red dot in Fig.~\ref{fig:preview_concept_visualization}). Finally, we choose the needle with its scaled tip closest to the target as the chosen needle (blue needle in Fig.~\ref{fig:preview_concept_visualization}) and its scaled tip as the preview control target. The pseudo-code of this process is given in Algorithm~\ref{alg:preview}. 

\subsection{Complete Pipeline}
\label{sec:pipeline}
\begin{algorithm}[t!]
\caption{\textit{\textbf{Mariner}}: CBF-based Safe Preview Control}
\begin{algorithmic}[1]
    \State Get point cloud data $\mathcal{P}$, number of needles $n_\text{needle}$, a distribution function $\mathrm{dist}$, and current global planner target $\bm{p}^\star$;
    \For{$i = 0\ \text{to}\ n_\text{needle} - 1$}
        \State $\theta_{\text{needle}, i} = \mathrm{dist}(i)$;
        \State Transform and filter $\mathcal{P}$ based on~\eqref{eq:transform_pcd} and~\eqref{eq:md};
        \State Compute the scale of this needle $\bar{\bm{s}}_i$;
    \EndFor
    \State Find the scaled needle with its tip closest to $\bm{p}^\star$;
    \State Set the preview control target to be the tip of the chosen needle from the last step; 
\end{algorithmic}
\label{alg:preview}
\end{algorithm}
The complete navigation pipeline includes the global planner, the Mariner, and the Vessel. Given a map of the environment, the global planner creates waypoints $\{\bm{p}_{k}^\star\}_{k = 1, \cdots, n_k} \subset \mathbb{R}^3$ between the current position of the robot and the target position, with $n_k \in \mathbb{N}$ being the number of waypoints. As mentioned in Section~\ref{sec:problem_formulation}, the map does not need to include transient structures. Then, the Mariner determines the local target ${}^\mathcal{B}\bm{p}_\text{target}^\star$. Finally, a performance controller is used to command the robot toward the local target, and the Vessel filters the performance controller command to ensure safety. Note that any sampling-based~\cite{KaramanF11} or search-based motion planning algorithm can be deployed for the global planner. For details regarding examples of the global planners that could be deployed, please refer to Section~\ref{sec:experiments}. The complete pipeline is also illustrated in Fig.~\ref{fig:complete_pipeline}.
\section{Experiments}
\label{sec:experiments}
This section shows the effectiveness of our method by comparing it against existing point cloud-based CBF formulations and validating its performance in real-world settings. 

\begin{figure}[t!]
    \centering
    \vspace{0.3cm}
    \includegraphics[width=0.45\textwidth]{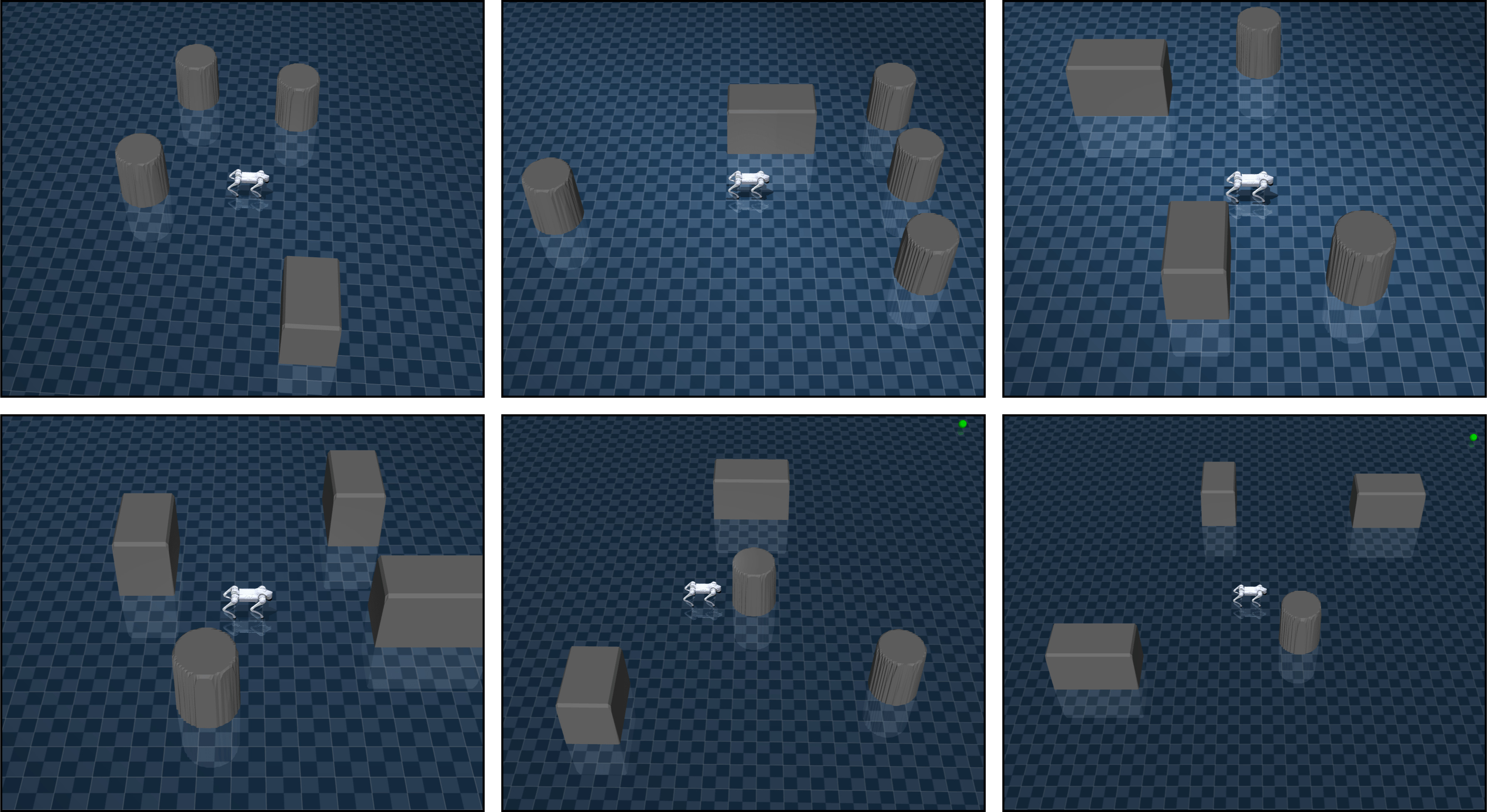}
    \caption{Examples of the environments used for the comparison studies.}
    \label{fig:sim_example}
\end{figure}

\subsection{Setup}
We validate our algorithm using Unitree B1 and Go2 robots. The simulated comparison studies are performed in MuJoCo with a Unitree Go2. All computations for real-world experiments are performed onboard using an NVIDIA Jetson Orin NX for the Unitree Go2 and a Jetson Orin AGX for the Unitree B1. Both robots obtain the point clouds using a 16-channel 3D LiDAR at 10~\si{Hz}. In the remainder of this section, the CBFQP models the dynamics of the quadruped system as a single integrator with orientation~\cite{CosnerRMUYAB22, UnluGCTK24}, i.e.,
\begin{equation}
\label{eq:quadruped_high_level_dynamics}
    \dot{\bm{x}} = \bm{u} = \begin{bmatrix}
        \bm{v}^\top & \bm{\omega}
    \end{bmatrix}^\top \in \mathbb{R}^4
\end{equation}
with $\bm{v}\in\mathbb{R}^3$ and $\bm{\omega}\in\mathbb{R}$ being the linear and angular velocity, respectively. The performance controller has the form of
\begingroup
\allowdisplaybreaks
\begin{subequations}
\label{eq:performance_controller}
\begin{align}
    \bm{u}_\mathrm{ref} &= \begin{bmatrix}
        \bm{v}^\top & \bm{\omega}    
    \end{bmatrix}^\top \in \mathbb{R}^4\\
    \bm{v} &= \bm{K}_v({}^\mathcal{B}\bm{p}_\text{target}^\star - {}^\mathcal{B}\bm{p}_\text{robot})\\
    \bm{\omega} &= \bm{K}_\omega({}^\mathcal{B}\bm{\phi}_\text{target} - {}^\mathcal{B}\bm{\phi}_\text{robot})
\end{align}
\end{subequations}
\endgroup
with $\bm{K}_v \in \mathbb{R}_{\succ0}^{3\times3}$ and $\bm{K}_\omega\in\mathbb{R}_+$ representing the gain matrices, ${}^\mathcal{B}\bm{\phi}_\text{target} \in \mathbb{R}$ the desired yaw angle which points towards the target, ${}^\mathcal{B}\bm{p}_\text{robot}\in\mathbb{R}^3$ the robot position in its body frame (is constantly zero), and ${}^\mathcal{B}\bm{\phi}_\text{robot}\in\mathbb{R}$ the robot yaw angle measured in its body frame (also is constantly zero). 
\begin{rem}
    The modeling of mobile robotic systems using velocity control is common in the literature~\cite{JianYLLLWL23, UnluGCTK24, CosnerRMUYAB22, GoncalvesKTK24, GoncalvesCTK24}. The accurate modeling of quadrupeds is outside the scope of this paper.
\end{rem}
For the CBFQP, we choose
\begin{equation}
    \Gamma(\bm{h}(\bm{x})) = \bar{\gamma}\bm{h}(\bm{x})
\end{equation}
where $\bar{\gamma} \in \mathbb{R}_+$. In the simulation studies, we use a learned locomotion policy~\cite{MargolisA22} to transform the commanded linear and angular velocity to joint actions. We use the robot's onboard proprietary controller in the experimental studies to track the commanded linear and angular velocity. For real-world experiments, the SLAM toolbox~\cite{MacenskiJ21} is used for mapping and localization. The Mariner and Vessel GPU kernels are created using NVIDIA \texttt{Warp}~\cite{warp2022}. Specifically, we create a 1D kernel for the Vessel and a 2D kernel for the Mariner. We show the computation time for an NVIDIA Jetson Orin NX and NVIDIA Jetson Orin AGX in Fig.~\ref{fig:computation_time}. In all the cases shown in Fig.~\ref{fig:computation_time}, we set $n_\mathrm{needle} = 100$. Additional experimental results, including the real-world experiments using Unitree B1, can be found in the accompanying video: \video.

\begin{table*}[t!]
\centering
\vspace{0.3cm}
\begin{threeparttable}[b]
\caption{Comparison Study Results\tnote{1}}
\label{tab:comparison}
\begin{tabular}[width=\textwidth]{m{6cm}<{\raggedright}m{1cm}<{\raggedright}m{1cm}<{\raggedright}m{1cm}<{\raggedright}m{1.5cm}<{\raggedright}m{1.5cm}<{\raggedright}m{2cm}<{\raggedright}}
\toprule
\textbf{Experiment} & $\hat{\bm{\mu}}_l$ $\downarrow$ & $\hat{\bm{\mu}}_s$ $\downarrow$ & $\hat{\bm{\mu}}_d$ $\downarrow$ & $\bm{\mu}_{t_\mathrm{CBF}}[\si{ms}]\ \downarrow$ & $\bm{\mu}_{t_\mathrm{LP}}[\si{ms}]\ \downarrow$ & \textbf{Success Rate} $\uparrow$\\
\midrule
\textit{Ours (Vessel +  Mariner)} & \textbf{1.00} & \textbf{1.00} & \textbf{1.00} & 0.78 & \textbf{0.37} & \textbf{100\%}\\
\textit{CBF with NMPC}~\cite{JianYLLLWL23} & 1.11 & 4.60 & 1.33 & 7.68 & 174.86 & \textbf{100\%}\\
\textit{Signed Distance CBF with Circulation}~\cite{UnluGCTK24} & 1.11 & 4.37 & 1.53 & \textbf{0.65} & N/A\tnote{2} & \textbf{100\%}\\
\textit{Our Mariner + CBF in}~\cite{JianYLLLWL23} & 1.01 & 1.68 & 1.46& 6.6 & 0.51 & \textbf{100\%}\\
\textit{Our Mariner + Signed Distance CBF}~\cite{UnluGCTK24} & 1.01 & \textbf{1.00} & 1.28 & 0.66 & 0.47 & \textbf{100\%}\\
\bottomrule
\end{tabular}
\begin{tablenotes}
    \item[1] $\downarrow$ means the lower that metric, the better the performance, and vice versa. Bold numbers represent the metric-wise best performance.
    \item[2] The N/A is due to the method in~\cite{UnluGCTK24} does not utilize a local planner.
\end{tablenotes}
\end{threeparttable}
\end{table*}

\begin{figure*}[t!]
    \centering
    \includegraphics[width=0.9\textwidth]{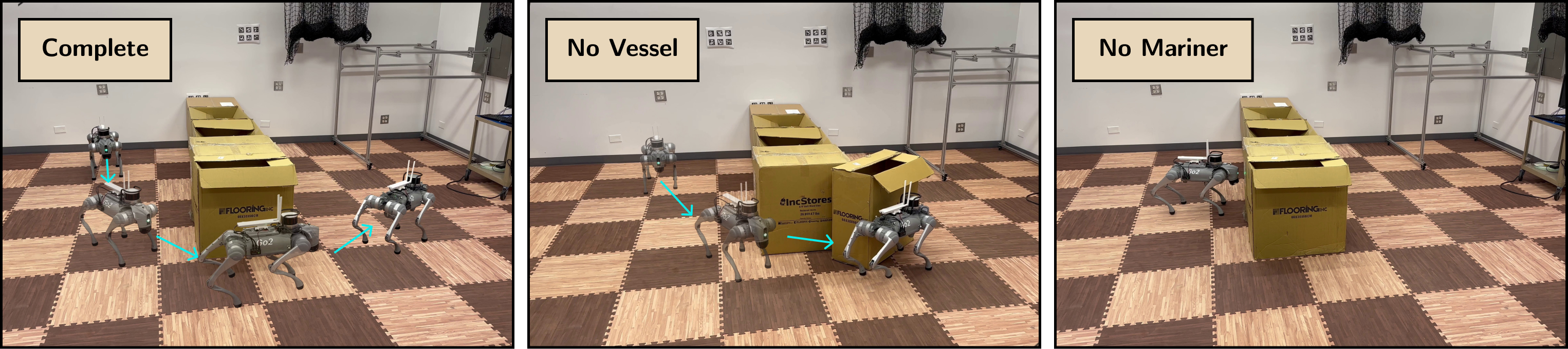}
    \caption{The experimental setting and ablation study results for showing the effectiveness of each component of the local planner where ``Complete" refers to having both the Mariner and Vessel, ``No Vessel" refers to only using the Mariner, and ``No Mariner" refers to only using the Vessel.}
    \label{fig:exp_settings}
\end{figure*}

\subsection{Simulation Comparison Studies}
We compare our proposed local planner with two baselines~\cite{JianYLLLWL23, UnluGCTK24} and with the CBFs from~\cite{JianYLLLWL23, UnluGCTK24} used along the Mariner on 50 different environments with randomly positioned obstacles. Since our proposed approach is a local planner, we only compare against the local planners in~\cite{JianYLLLWL23, UnluGCTK24}. In all comparison studies, the LiDAR measurements come at 10~\si{Hz}, and each scan has 1024 points. For all experiments, the robot starts at $(-4, -4)~\si{m}$ facing towards the positive $x$ direction and the goal is at $(5, 5)~\si{m}$. Examples of the environments are shown in Fig.~\ref{fig:sim_example}. Since the difficulty of each environment varies, the path length, smoothness, and minimal distance along the path are only compared relatively, i.e., $\text{metric}_\text{others} / \text{metric}_\text{ours}$ is given. In summary, we compare the performance of each of the methods using five metrics: the mean relative path length $\hat{\bm{\mu}}_l$, mean relative smoothness $\hat{\bm{\mu}}_s$ (smoothness is represented as the mean curvature along the path), mean relative minimal distance to obstacles along the path $\hat{\bm{\mu}}_d$, mean CBF computation time $\bm{\mu}_{t_\mathrm{CBF}}$, and mean local planner computation time $\bm{\mu}_{t_\mathrm{LP}}$. The comparison studies result is shown in Table~\ref{tab:comparison}. When compared against~\cite{JianYLLLWL23}, the minimum bounding ellipsoid (MBE) based CBF used by~\cite{JianYLLLWL23} cannot be generated reliably and is computationally much slower compared to the Vessel. Furthermore, the generated ellipsoid occasionally collides with the sphere modeling the robot, which results in the MPC starting in an infeasible state. In our comparison studies, we created safeguards such as reusing the old MPC solution or setting the velocity to zero when the MPC fails to find a solution. However, this might still lead to unsafe scenarios. Compared to our method, these issues make~\cite{JianYLLLWL23} computationally less efficient and stable. When comparing against~\cite{UnluGCTK24}, we noticed that~\cite{UnluGCTK24} generates longer and curvier paths, while the CBF computation time is similar. Additionally, the intermediate computations for the CBF in~\cite{UnluGCTK24} can generate very large numbers, which might cause numerical issues. We also tested combining our Mariner with the CBFs proposed in~\cite{JianYLLLWL23} and~\cite{UnluGCTK24}. In both cases, the inclusion of the Mariner leads to improved performance. Therefore, when compared against baseline approaches, our proposed local planner is computationally efficient, robust to a wide range of point cloud measurements, and can generate shorter and more direct paths. Additionally, spheres~\cite{JianYLLLWL23} and smoothed rectangles~\cite{UnluGCTK24} can both be modeled using the Vessel (see Fig.~\ref{fig:cbf_concept_visualization}).

\subsection{Experimental Deployment}
\label{sec:local_planner}
This section presents the experimental results of deploying our algorithm on a Unitree Go2 robot. The goal is to reach a target location while avoiding collision with obstacles (the boxes in Fig.~\ref{fig:exp_settings}). The semi-axes $(a, b, c)$ for the Vessels are $(0.5, 0.3, 0.2)~\si{m}$. The Needles are modeled using ellipsoids with semi-axes $(0.8, 0.1, 0.2)~\si{m}$, $\bar{\bm{d}}=2$, $n_\text{needle} = 100$, and the distribution function in~\eqref{eq:dist_func}. The generated trajectories are shown in Fig.~\ref{fig:exp_settings}. The robot successfully reaches the target position on the other side of the box wall when using both the Vessel and Mariner (denoted as complete in Fig.~\ref{fig:exp_settings}). The robot collides with the boxes if the Vessel is not activated. If the Mariner is not activated, the robot fails to escape from the equilibrium of the CBFQP controller. For results of our proposed pipeline on the Unitree B1 robot in a similar setting, please refer to the accompanying \href{https://youtu.be/P9NPv1f3kXQ}{video}. Next, we show the results of integrating our proposed local planner with a global planner. The RRT$^{\star}$ implementation in OMPL~\cite{SucanMK12} is used to generate a global plan, which is shown in Fig.~\ref{fig:complete_pipeline}. Then, random obstacles (four boxes as shown in Fig.~\ref{fig:global_planner}) are added to the scene. Finally, following the pipeline shown in Section~\ref{sec:pipeline}, the robot tracks the RRT-generated waypoints. The Mariner and Vessel settings are identical to the experiments in Section~\ref{sec:local_planner}. As shown in Fig.~\ref{fig:global_planner}, the local planner generates collision-free paths for the robot without getting stuck at any point along the trajectory. Additional experiments are in the accompanying \href{https://youtu.be/P9NPv1f3kXQ}{video}. For these additional experiments, the new obstacles were also not present during the mapping and planning stages. 
\begin{figure*}[t!]
    \centering
    \includegraphics[width=0.85\textwidth]{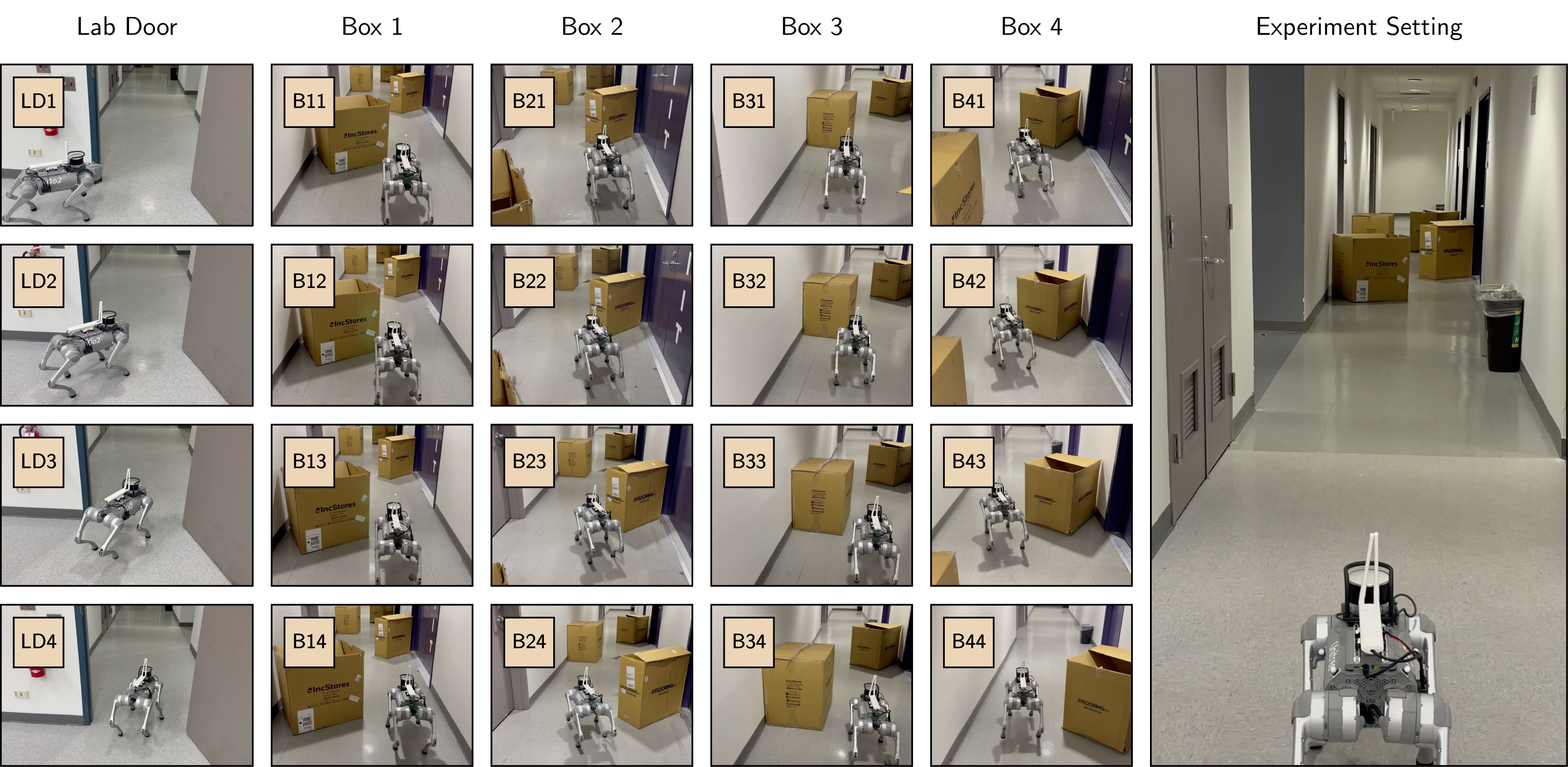}
    \caption{Result of our complete pipeline. There are five main obstacles in the illustrated settings. For each obstacle, four snapshots are provided to depict the motion generated by the proposed local planner.}
    \label{fig:global_planner}
\end{figure*}
\section{Limitations}
\label{sec:limitations}

The main limitations of our proposed local planner are as follows: (1) although the local planner greatly reduces the probability of getting stuck at spurious equilibria, it does not guarantee that such spurious equilibria never occur; (2) the proposed method does not explicitly consider input constraints, which could be potentially important to address in some applications. We plan to address these issues in future works.

\section{Conclusion}
\label{sec:conclusion}

In this work, we proposed \textbf{\textit{Vessel}} a novel point cloud based CBF formulation and \textbf{\textit{Mariner}} a novel local planner that can process point cloud data and move the CBF-based controller out of local equilibria. We have shown theoretically that the proposed CBF is valid. To show the effectiveness of our approach, we have tested it on quadrupedal robots in both simulation and the real world, along with the proposed motion planning pipeline. 

\bibliographystyle{IEEEtran}
\bibliography{IEEEabrv, refs.bib}

\begin{thebibliography}{10}
\providecommand{\url}[1]{#1}
\csname url@samestyle\endcsname
\providecommand{\newblock}{\relax}
\providecommand{\bibinfo}[2]{#2}
\providecommand{\BIBentrySTDinterwordspacing}{\spaceskip=0pt\relax}
\providecommand{\BIBentryALTinterwordstretchfactor}{4}
\providecommand{\BIBentryALTinterwordspacing}{\spaceskip=\fontdimen2\font plus
\BIBentryALTinterwordstretchfactor\fontdimen3\font minus \fontdimen4\font\relax}
\providecommand{\BIBforeignlanguage}[2]{{%
\expandafter\ifx\csname l@#1\endcsname\relax
\typeout{** WARNING: IEEEtran.bst: No hyphenation pattern has been}%
\typeout{** loaded for the language `#1'. Using the pattern for}%
\typeout{** the default language instead.}%
\else
\language=\csname l@#1\endcsname
\fi
#2}}
\providecommand{\BIBdecl}{\relax}
\BIBdecl

\bibitem{JianYLLLWL23}
Z.~Jian, Z.~Yan, X.~Lei, Z.~Lu, B.~Lan, X.~Wang, and B.~Liang, ``Dynamic control barrier function-based model predictive control to safety-critical obstacle-avoidance of mobile robot,'' in \emph{Proceedings of the {IEEE} International Conference on Robotics and Automation, London, United Kingdom}, May 2023, pp. 3679--3685.

\bibitem{UnluGCTK24}
H.~U. Unlu, V.~M. Gon{\c{c}}alves, D.~Chaikalis, A.~Tzes, and F.~Khorrami, ``A control barrier function-based motion planning scheme for a quadruped robot,'' in \emph{Procceedings of {IEEE} International Conference on Robotics and Automation, Yokohama, Japan}, May 2024, pp. 12\,172--12\,178.

\bibitem{DaiKPK23}
B.~Dai, P.~Krishnamurthy, A.~Papanicolaou, and F.~Khorrami, ``State constrained stochastic optimal control for continuous and hybrid dynamical systems using {DFBSDE},'' \emph{Automatica}, vol. 155, p. 111146, 2023.

\bibitem{RawlingsMD17}
J.~B. Rawlings, D.~Q. Mayne, and M.~Diehl, \emph{Model predictive control: theory, computation, and design}.\hskip 1em plus 0.5em minus 0.4em\relax Nob Hill Publishing, Madison, WI, 2017, vol.~2.

\bibitem{AmesCENST19}
A.~D. Ames, S.~Coogan, M.~Egerstedt, G.~Notomista, K.~Sreenath, and P.~Tabuada, ``Control barrier functions: Theory and applications,'' in \emph{Proceedings of the European Control Conference, Naples, Italy}, June 2019, pp. 3420--3431.

\bibitem{GoncalvesKTK24}
V.~M. Gonçalves, P.~Krishnamurthy, A.~Tzes, and F.~Khorrami, ``Control barrier functions with circulation inequalities,'' \emph{IEEE Transactions on Control Systems Technology}, pp. 1--16, 2024.

\bibitem{ThirugnanamZS22}
A.~Thirugnanam, J.~Zeng, and K.~Sreenath, ``Safety-critical control and planning for obstacle avoidance between polytopes with control barrier functions,'' in \emph{Proceedings of International Conference on Robotics and Automation, Philadelphia, PA}, May 2022, pp. 286--292.

\bibitem{ZengZS21}
J.~Zeng, B.~Zhang, and K.~Sreenath, ``Safety-critical model predictive control with discrete-time control barrier function,'' in \emph{Proceedings of the American Control Conference, New Orleans, LA}, May 2021, pp. 3882--3889.

\bibitem{LaValle06}
S.~M. LaValle, \emph{Planning Algorithms}.\hskip 1em plus 0.5em minus 0.4em\relax Cambridge University Press, Cambridge, United Kingdom, 2006.

\bibitem{LiuWMSBTK17}
S.~Liu, M.~Watterson, K.~Mohta, K.~Sun, S.~Bhattacharya, C.~J. Taylor, and V.~Kumar, ``Planning dynamically feasible trajectories for quadrotors using safe flight corridors in 3-d complex environments,'' \emph{{IEEE} Robotics and Automation Letters}, vol.~2, no.~3, pp. 1688--1695, 2017.

\bibitem{DaiKK22}
B.~Dai, P.~Krishnamurthy, and F.~Khorrami, ``Learning a better control barrier function,'' in \emph{Proceedings of the {IEEE} Conference on Decision and Control, Canc\'{u}n, Mexico}, December 2022, pp. 945--950.

\bibitem{DaiHKK23}
B.~Dai, H.~Huang, P.~Krishnamurthy, and F.~Khorrami, ``Data-efficient control barrier function refinement,'' in \emph{Proceedings of the American Control Conference, San Diego, CA}, May 2023, pp. 3675--3680.

\bibitem{DaiKKGTK23}
B.~Dai, R.~Khorrambakht, P.~Krishnamurthy, V.~Gon{\c{c}}alves, A.~Tzes, and F.~Khorrami, ``Safe navigation and obstacle avoidance using differentiable optimization based control barrier functions,'' \emph{{IEEE} Robotics and Automation Letters}, vol.~8, no.~9, pp. 5376--5383, 2023.

\bibitem{DaiKKK23}
B.~Dai, R.~Khorrambakht, P.~Krishnamurthy, and F.~Khorrami, ``Differentiable optimization based time-varying control barrier functions for dynamic obstacle avoidance,'' \emph{CoRR}, vol. abs/2309.17226, 2023.

\bibitem{WeiDKKK24}
S.~Wei, B.~Dai, R.~Khorrambakht, P.~Krishnamurthy, and F.~Khorrami, ``Diffocclusion: Differentiable optimization based control barrier functions for occlusion-free visual servoing,'' \emph{IEEE Robotics and Automation Letters}, vol.~9, no.~4, pp. 3235--3242, 2024.

\bibitem{OngG96}
C.~J. Ong and E.~G. Gilbert, ``Growth distances: new measures for object separation and penetration,'' \emph{{IEEE} Transactions on Robotics and Automation}, vol.~12, no.~6, pp. 888--903, 1996.

\bibitem{GilbertO94}
E.~G. Gilbert and C.~J. Ong, ``New distances for the separation and penetration of objects,'' in \emph{Proceedings of the {IEEE} International Conference on Robotics and Automation, San Diego, CA}, May 1994, pp. 579--586.

\bibitem{TracyHM22}
K.~Tracy, T.~A. Howell, and Z.~Manchester, ``Differentiable collision detection for a set of convex primitives,'' in \emph{Proceedings of the {IEEE} International Conference on Robotics and Automation, London, United Kingdom}, May 2023, pp. 3663--3670.

\bibitem{ZhangGF23}
S.~Zhang, K.~Garg, and C.~Fan, ``Neural graph control barrier functions guided distributed collision-avoidance multi-agent control,'' in \emph{Proceedings of Conference on Robot Learning, Atlanta, GA}, vol. 229, November 2023, pp. 2373--2392.

\bibitem{CosnerRMUYAB22}
R.~K. Cosner, I.~D.~J. Rodriguez, T.~G. Moln{\'{a}}r, W.~Ubellacker, Y.~Yue, A.~D. Ames, and K.~L. Bouman, ``Self-supervised online learning for safety-critical control using stereo vision,'' in \emph{Proceedings of International Conference on Robotics and Automation, Philadelphia, PA}, May 2022, pp. 11\,487--11\,493.

\bibitem{SingletaryKBBTA21}
A.~Singletary, K.~Klingebiel, J.~Bourne, N.~A. Browning, P.~Tokumaru, and A.~D. Ames, ``Comparative analysis of control barrier functions and artificial potential fields for obstacle avoidance,'' in \emph{Proceedings of {IEEE/RSJ} International Conference on Intelligent Robots and Systems, Prague, Czech Republic}, September 2021, pp. 8129--8136.

\bibitem{FoxBT96}
D.~Fox, W.~Burgard, and S.~Thrun, ``Controlling synchro-drive robots with the dynamic window approach to collision avoidance,'' in \emph{Proceedings of the {IEEE/RSJ} International Conference on Intelligent Robots and Systems, Osaka, Japan}, November 1996, pp. 1280--1287.

\bibitem{KaramanF11}
S.~Karaman and E.~Frazzoli, ``Sampling-based algorithms for optimal motion planning,'' \emph{International Journal of Robotics Research}, vol.~30, no.~7, pp. 846--894, 2011.

\bibitem{GoncalvesCTK24}
V.~M. Gon{\c{c}}alves, D.~Chaikalis, A.~Tzes, and F.~Khorrami, ``Safe multi-agent drone control using control barrier functions and acceleration fields,'' \emph{Robotics and Autonomous Systems}, vol. 172, p. 104601, 2024.

\bibitem{MargolisA22}
G.~B. Margolis and P.~Agrawal, ``Walk these ways: Tuning robot control for generalization with multiplicity of behavior,'' in \emph{Proceedings of Conference on Robot Learning, Auckland, New Zealand}, vol. 205, December 2022, pp. 22--31.

\bibitem{MacenskiJ21}
S.~Macenski and I.~Jambrecic, ``{SLAM} toolbox: {SLAM} for the dynamic world,'' \emph{Journal of Open Source Software}, vol.~6, no.~61, p. 2783, 2021.

\bibitem{warp2022}
M.~Macklin, ``Warp: A high-performance {P}ython framework for {GPU} simulation and graphics,'' March 2022, {NVIDIA GPU} Technology Conference.

\bibitem{SucanMK12}
I.~A. Sucan, M.~Moll, and L.~E. Kavraki, ``The open motion planning library,'' \emph{{IEEE} Robotics \& Automation Magazine}, vol.~19, no.~4, pp. 72--82, 2012.

\end{thebibliography}
\end{document}